\documentclass[letterpaper]{article} 
\usepackage[submission]{aaai23}  
\usepackage{times}  
\usepackage{helvet}  
\usepackage{amsthm}
\usepackage{amsmath}
\newtheorem{definition}{Definition}  
\newtheorem{theorem}{Theorem}
\newtheorem{corollary}{Corollary}
\usepackage{courier}  
\usepackage[hyphens]{url}  
\usepackage{graphicx} 
\usepackage{natbib}  
\usepackage{caption} 
\usepackage{amsfonts}
\usepackage{algorithm}
\usepackage{algorithmic}

\usepackage{newfloat}
\usepackage{listings}
\DeclareCaptionStyle{ruled}{labelfont=normalfont,labelsep=colon,strut=off} 
\floatstyle{ruled}
\newfloat{listing}{tb}{lst}{}
\floatname{listing}{Listing}
\title{Approximate, Adapt, Anonymize (3A): a Framework for Privacy Preserving Training Data Release for Machine Learning}
\author{
    Tamas Madl\equalcontrib,
    Weijie Xu,
    Olivia Choudhury,
    Matthew Howard\equalcontrib\\
}
\affiliations{
    \textsuperscript{\rm 1}Association for the Advancement of Artificial Intelligence\\

    1900 Embarcadero Road, Suite 101\\
    Palo Alto, California 94303-3310 USA\\
    publications23@aaai.org
}

\usepackage{bibentry}

\begin{document}

\maketitle

\begin{abstract}
The availability of large amounts of informative data is crucial for successful machine learning. However, in domains with sensitive information, the release of high-utility data which protects the privacy of individuals has proven challenging. Despite progress in differential privacy and generative modeling for privacy-preserving data release in the literature, only a few approaches optimize for machine learning utility: most approaches only take into account statistical metrics on the data itself and fail to explicitly preserve the loss metrics of machine learning models that are to be subsequently trained on the generated data.
In this paper, we introduce a data release framework, 3A (Approximate, Adapt, Anonymize), to maximize data utility for machine learning, while preserving differential privacy. The framework aims to 1) learn an approximation of the underlying data distribution, 2) adapt it such that loss metrics of machine learning models are preserved as closely as possible, and 3) anonymize by using a noise addition mechanism to ensure differential privacy. We also describe a specific implementation of this framework that leverages mixture models to approximate, kernel-inducing points to adapt, and Gaussian differential privacy to anonymize a dataset, in order to ensure that the resulting data is both privacy-preserving and high utility.
We present experimental evidence showing minimal discrepancy between performance metrics of models trained on real versus privatized datasets, when evaluated on held-out real data. We also compare our results with several privacy-preserving synthetic data generation models (such as differentially private generative adversarial networks), and report significant increases in classification performance metrics compared to state-of-the-art models. These favorable comparisons show that the presented framework is a promising direction of research, increasing the utility of low-risk synthetic data release for machine learning.

\end{abstract}

\section{Introduction}

Training and using machine learning models in domains with sensitive and personally identifiable information such as healthcare presents significant legal, ethical and trust challenges; slowing the progress of this technology as well as its potential positive impact. Data synthesis has been proposed as a potential mechanism that can be a legally and ethically appropriate solution to the sharing and processing of sensitive data. It is possible to synthesize a dataset that allows accurate machine learning (ML) model training without compromising privacy, depending on privacy definition and requirements. In the case of the GDPR, for example, data that do not allow singling out may be excepted from the regulation \cite{cohen2020towards}. In this work, we will rely on differential privacy \cite{dwork2006differential}, which unlike weaker privacy criteria such as k-anonymity has been shown to protect against singling out individuals \cite{cohen2020towards}.

\subsection{A privacy-preserving synthetic data generation framework maximizing utility for machine learning}

For the purposes of privacy-preserving ML by means of differentially private (DP) data release, we are interested in approaches which 1) approximate the true data distribution, 2) preserve utility for machine learning (ML models trained on the data release perform similarly to models trained on true data), and 3) preserve privacy in accordance to DP.

More formally, we propose studying a class of data generation algorithms $M$ which, given an original dataset $D=(X_i,Y_i)_{i=1}^n$ with $n$ data points $X_i$ and labels $Y_i$, produce a synthetic dataset $\tilde{D}=M(D)$, such that they
\begin{enumerate}
    \item \textbf{Approximate} the underlying data distribution: estimate a parametric density $p_\theta(x)$ by optimizing a log-likelihood objective $L_1\left(p_{\theta}, D\right):=\mathbb{E}_{x \sim D}\left[-\log p_{\theta}(x)\right]$
    \item \textbf{Adapt} the approximated data distribution such that the loss of a classifier $f$ trained on data sampled from it is close to the loss of a classifier $\tilde{f}$ on the original data, under loss $l$:
    \\
  \[
  L_2(D, \tilde{D}, f, \tilde{f}) = \left|\mathbb{E}_{(x, y) \sim D}[\ell(f(x), y)]-\mathbb{E}_{(x, y) \sim \tilde{D}}[\ell(\tilde{f}(x), y)]\right|
  \]

    The overall optimization procedure needs to trade off the importance of $L_1$ the objective encouraging faithfully preserving the data distribution, and of $L_2$, the objective encouraging matching classifier loss: $L = \alpha L_1 + (1-\alpha) L_2$. 
    \item \textbf{Anonymize} by ensuring $(\epsilon, \delta)$ differential privacy of the overall data publishing algorithm, such that the participation of a single data point is unlikely to be distinguishable. That is, ensure the data publishing algorithm is differentially private \cite{dwork2006differential}.
    
\end{enumerate}
\begin{figure*}[t]
  \includegraphics[width=\textwidth]{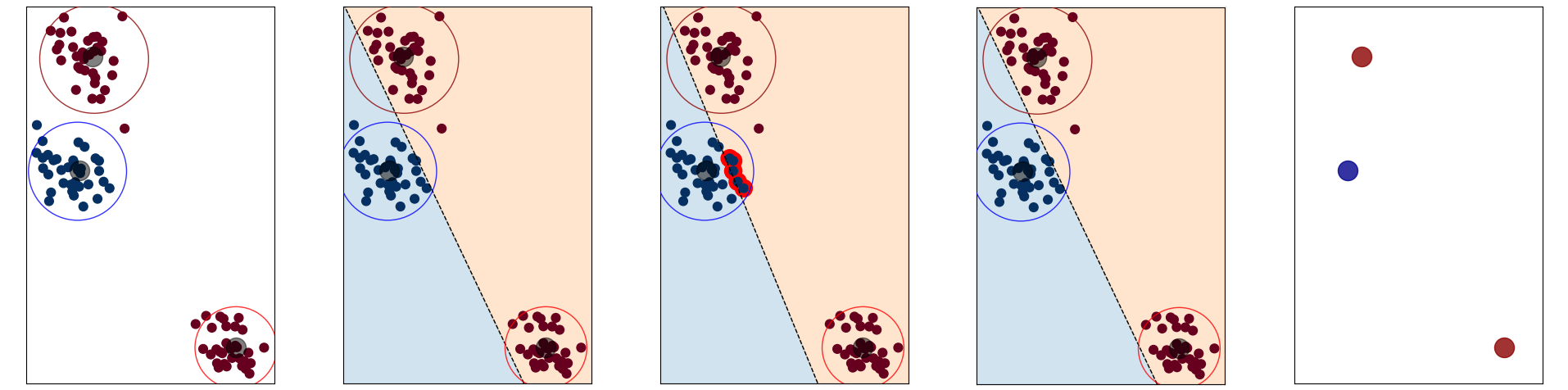}
  \caption{Illustration of the 3A (Approximate, Adapt, Anonymize) privacy-preserving data generation framework applied in a simple linearly separable setting. Panel 1 (\textit{Approximate}): density estimation (in our case, Gaussian Mixture Models). Panel 2: baseline classifier trained on full data (dashed black line). Panel 3: classifier trained on sample from the density estimator (dashed black line), which misclassifies several data points close to the blue cluster (marked with thick red borders). Panel 4 (\textit{Adapt}): adapted model which avoids these misclassifications (in our case, using Kernel Inducing Points \cite{nguyen2020dataset} and trading off density vs. classifier fidelity). Panel 5 (\textit{Anonymize}): differentially private dataset facilitating the same accuracy as the baseline classifier and minimizing identifiability risk by taking convex combinations of nearby data points (`safety in numbers') and adding noise according to Gaussian Differential Privacy \cite{dong2019gaussian}}
\end{figure*}

Many instantiations of this general framework are possible. In this paper, we evaluate ClustMix, a simple algorithm instantiating these 3 steps. We will choose 1) a Gaussian Mixture Model as the density estimator \cite{bishop2006pattern}, 2) the Kernel Inducing Point meta-learning algorithm \cite{nguyen2020dataset} as the loss approximator (with an objective allowing a tradeoff between preserving density fidelity vs. preserving classifier fidelity), and 3) an improved version of Random Mixing to ensure privacy \cite{lee2018sgd,lee2019synthesizing} - preserving combinations of data points, rather than individual data points, to facilitate a `safety in numbers' approach to avoiding reidentification. Our main contributions are the flexible privacy-preserving data generation framework described above, and the introduction of cluster-based instead of random mixing for preserving differential privacy, which, together, allow significant accuracy increases over previously published methods (see Table \ref{tab:results}).

The idea to create new training examples by taking convex combinations of existing data points has been successfully leveraged in machine learning, e.g. for data augmentation \cite{zhang2017mixup,inoue2018data}, learning with redundancy in distributed settings \cite{karakus2017straggler}, and more recently also for private machine learning \cite{lee2018sgd}. Lee et al. \cite{lee2019synthesizing} leverage random mixtures (convex combinations of a randomly sampled subset of a dataset) and additive Gaussian noise to design a differentially private (DP) data release mechanism.

However, most of these methods ignore data geometry, sampling at random instead of explicitly attempting to preserve the original data distribution. This may lead to lower downstream utility for machine learning, as low-density regions around decision boundaries may not be preserved. Mixtures of random samples may also fail to preserve certain data distributions, such as skewed and multimode continuous variables. 

In our approach, instead of random sampling, we sample from the immediate neighborhood of cluster centroids in order to preserve data distribution. Focusing on preserving cluster structure and mixing similar data points instead of random data points allows noisy mixtures to more closely approximate the original data distribution, and lose less utility than competing methods despite stronger DP guarantee.

\section{Related Work}

Privacy-preserving data publishing for machine learning tasks is an important problem, and several approaches have been proposed to address it. Mechanisms for publishing such data can operate in an interactive or non-interactive setting~\cite{zhu2017differentially}. In an interactive setting, the data publishing method receives queries from users and responds with noisy outputs to preserve data privacy. The performance of interactive data publishing methods depend on several constraints, such as type of query, maximum number of queries, accuracy, and computational efficiency. Malicious users may still be able to deduce an output much closer to the original answer, thereby exposing the sensitivity of the dataset. In a non-interactive setting, the publishing method releases the data all at once in a privacy-preserving manner. 

Among the state-of-the-art non-interactive approaches, synthetic data publishing is the most widely-used technique. This can be achieved by either anonymizing the dataset~\cite{mohammed2011differentially} or generating new samples based on the original data distribution~\cite{zhang2017privbayes,blum2013learning,kasiviswanathan2011can,hardt2012simple}. As noted in~\cite{lee2019synthesizing}, computational complexity of most of the existing methods is exponential in the dimensionality of the dataset. Hence, they are not applicable to deep learning applications that commonly deal with high-dimensional data.

An alternate method is local perturbation that perturbs every data point with additive noise~\cite{agrawal2000privacy,mishra2006privacy}. A relevant approach is random projection, which extracts lower-dimensional features via random projection and then perturbs them with some noise~\cite{kenthapadi2012privacy,xu2017dppro}. In~\cite{lee2019synthesizing}, the authors propose a new data publishing algorithm  called Differentially Private Mix (DPMix), which generates a dataset by mixing $\ell$ randomly chosen data points and then perturbing them with an additive noise.

Generative Adversarial Network (GAN)~\cite{goodfellow2014generative} is a well-known method for generating synthetic data from real data. As GANs do not provide any privacy guarantees by default, several methods have been proposed to modify it. PATE-GAN~\cite{jordon2018pate} adapts the training procedure of the discriminator to be differentially private by using a modified version of the Private Aggregation of Teacher Ensembles (PATE) framework~\cite{papernot2016semi,papernot2018scalable}. The Differentially Private GAN (DP-GAN)~\cite{xie2018differentially} aims to preserve privacy during training, by adding noise to the gradient of the Wasserstein distance. More recently, Differentially Private Mean Embeddings (DP-MERF) with Random Features have been suggested and shown to provide substantial utility improvements. DP-MERF leverages random Fourier feature representations to approximate the maximum mean discrepancy (MMD) objective in terms of two finite-dimensional mean embeddings, detaching the approximated embedding of the true data from that of the synthetic data. The former is the only term that is data dependent and requires privatization only once (and can then be re-used repeatedly to train a generator model).

In the Results section, we evaluate against these methods, and add a non-DP method for completeness’s sake - ADS-GAN (Anonymization Through Data Synthesis Using Generative Adversarial Networks) \cite{yoon2020anonymization}. In addition to training a generator model that minimizes discrepancy to the real data (in this case using Wasserstein distance), the objective function of ADS-GAN also penalizes an identifiability term directly. ADS-GAN formulates identifiability as the percentage of synthetic data points appearing in closer proximity to a real data point than the next-nearest real neighbor of that real data point (in other words, synthetic instances which may put real instances at risk of  distance-to-closest-record attacks). 

\section{Methods}

We describe a simple instantiation of the 3A framework, dubbed ClustMix. We will choose 1) a Gaussian Mixture Model as the density estimator \cite{bishop2006pattern}, 2) the Kernel Inducing Point meta-learning algorithm \cite{nguyen2020dataset} as the loss approximator (with an objective allowing a tradeoff between preserving density fidelity vs. preserving classifier fidelity), and 3) an improved version of Random Mixing to ensure privacy \cite{lee2018sgd,lee2019synthesizing}.

ClustMix is presented in Algorithm 2. The mixing component leverages the same intuition from previous research in data augmentation \cite{zhang2017mixup,inoue2018data} and vicinal risk minimization \cite{chapelle2001vicinal} that convex combinations of samples from the vicinity of existing training examples tend to follow the original data distribution. 

\subsection{Approximation through Gaussian Mixture Models}

We consider modeling the data distribution from the probabilistic perspective, approximating the probability density of the data by means of a finite mixture model \citep{mclachlan2019finite}. The joint density over the training samples factorizes as

\begin{equation}
\label{eq:gmm}
p(X, Z \mid \theta)=\prod_{n=1}^{N} \prod_{k=1}^{K}\left[p\left(z_{n}=k\right) p\left(x_{n} \mid \theta_{k}\right)\right]^{\mathbb{I}\left[z_{n}=k\right]}
\end{equation}

In  ClustMix, before fitting the Gaussian Mixture Model, we randomly sliced data into n clusters and fit the model for each randomly sliced cluster. To be specific, for feature, we uniformly sample n data points as cutoff points. We apply the following algorithm to each small region. This makes sure each sample can only contribute to the creation of one cluster and creation of one synthetic sample.  We also make the assumption of isotropic Gaussian in order to reduce computational complexity, but note that the algorithm can be made more general by relaxing this constraint.


Due to the privacy-utility tradeoff associated with the size of mixtures - too small mixtures incur privacy risk, as shown by \cite{lee2019synthesizing}. The exact minimum mixture size $l^{min}(\epsilon,\delta,\sigma_{max})$ for given privacy parameters $(\epsilon,\delta)$ can be obtained from Eq. \ref{eq:epsilon} by means of optimization  

Thus, our data synthesis Algorithm 2 creates mixtures from clusters obtained from a Gaussian Mixture Model (GMM) adapted for size-constrained clustering, with the constraint that all cluster sizes have to be greater than or equal $l^{min}(\epsilon,\delta,\sigma_{max})$. We can reformulate Eq. \ref{eq:gmm} by factorizing to explicitly control the distribution over cluster sizes, as first proposed by \cite{jitta2018controlling}: 

\begin{equation}
\label{eq:dgmm}
p(X, Z \mid \theta)=\prod_{k=1}^{K}\left[p\left(s_{k}\right) \prod_{n=1}^{N} p\left(x_{n} \mid \theta_{k}\right)^{I\left[z_{n}=k\right]}\right],
\end{equation}

Where $s_k$ refers to the number of samples in cluster $k$. The assignment of data points to clusters can be obtained by finding the maximum of the joint log likelihood. X is the features and Z is latent distributions. 


\begin{align}
\label{eq:ll}
\log p(Z \mid X, \theta) &= \log p(X \mid \theta, Z)+\log p(Z) \\
&= \sum_{n} \log p\left(x_{n} \mid \theta_{z_{n}}\right)+\sum_{k} \log p\left(s_{k}\right)
\end{align}
$p(s)$ can be chosen to take the form of a step function, such that the probability is zero for all $s < l^{min}$ that are smaller than our minimum cluster size $l^{min}$, and uniformly nonzero otherwise. The optimum assignment based on Eq. \ref{eq:ll} can then be found by expectation maximization, and yields cluster parameters and cluster assignments that approximate the data distribution given the size constraint (all clusters must contain at least $l^{min}$ data points). We skip details for reasons of space (different solution approaches can be found in \cite{jitta2018controlling} and \cite{bishop2006pattern}). 

Below, we will use $\text{GMMConstrained}(\mathcal{X}, n, l^{min})$ to denote a function taking a dataset $\mathcal{X}$ and a minimum cluster size $l^{min}$ as inputs, and producing $n$ clusters (subsets of $\mathcal{X}$) as its output. 

\subsection{Adapt through Kernel Inducing Points}

The Kernel Inducing Points (KIP) meta-learning algorithm \cite{nguyen2020dataset} allows obtaining a smaller or distorted version of an original dataset which nevertheless maintains similar model performance as a model trained on the original data, and is thus uniquely well suited for the second step of our framework. It leverages kernel ridge regression, allowing a computationally efficient convex first order optimization approach to adaptation.

Let $D=(X_t,Y_t)_{i=1}^{n}$ be the original dataset, and $\tilde{D}=(\tilde{X_s},\tilde{Y_s})_{i=1}^{s} \sim p_\theta(x)$ be a dataset obtained from the approximated model. In our case, ClustMix uses the cluster centroids as the smallest number of most representative samples (since the amount of noise that needs to be added to protect privacy rapidly grows with sample size - see next section).

Given a kernel $K$, KIP defines a kernel ridge regression loss function as follows:

\begin{equation}
\label{eq:krr}
L_{KRR}\left(X_{s}, y_{s}\right)=\frac{1}{2}\left\|y_{t}-K_{X_{t} X_{s}}\left(K_{X_{s} X_{s}}+\lambda I\right)^{-1} y_{s}\right\|_{2}^{2}
\end{equation}

where $\lambda>0$ is a fixed regularization parameter. The KIP algorithm then minimizes Eq. \ref{eq:krr} with respect to the dataset $Xs$. The optimization procedure is initialized with $\tilde{X_s}$, the synthetic dataset obtained from the approximated model through extracting centroids. 

The kernel we use here is the Neural Tangent Kernel \cite{jacot2018neural}, mirroring the exact setup described in \cite{nguyen2021dataset} ($neural_tangents$ library, Adam optimizer, z-scoring as preprocessing), with two exceptions: omitting any data augmentation, and using a fully connected architecture instead of a convolutional neural network in order to facilitate application to tabular data.  

In our optimization procedure, we explicate a tradeoff between the importance of preserving data density and classifier accuracy. In addition to allowing users to make this choice explicitly (e.g. in domains where faithfully preserving density is important), this is also important for privacy, as allowing KIP to over-optimize on the loss may result in $Xs$ regimes which are difficult to anonymize. A simple example may be a data point in $Xs$ matching an outlier in $X$, thus easily exposing the identity of that outlier to simple privacy attacks. This can be avoided with a nonzero alpha in our procedure, which forces $Xs$ staying arbitrarily close to GMM cluster centroids. Finally, a nonzero $\alpha$ acts as a regularizer and prevents overfitting to a decision boundary that may not generalize well. 

\begin{equation}
\label{eq:kip}
  \begin{array}{l}
    L\left(X_{s}, y_{s}\right)=
    \alpha . \frac{\text{tr}((X_s - \tilde{X_s})^T (X_s - \tilde{X_s}))}{\text{tr}( \tilde{X_s}^T \tilde{X_s})}
    \\
    
    +(1-\alpha)\frac{1}{2}\left\|y_{t}-K_{X_{t} X_{s}}\left(K_{X_{s} X_{s}}+\lambda I\right)^{-1} y_{s}\right\|_{2}^{2}
   \end{array}
\end{equation}

The first term is the sum of normalized, $\alpha$-weighted Euclidean distances between the data points in the original approximation set $\tilde{Xs}$ and corresponding KIP-adapted data points $Xs$. Thus, $\alpha=0$ would preserve the output of the Approximate step (in our case, GMM cluster centroids), and $\alpha=1$ would simply yield the $Xs$ most conducive to preserving classifier accuracy (ignoring density approximation). In our implementation, we obtain the optimal $\alpha$ as part of hyperparameter optimization. We apply KIP in each randomly sliced clusters.

\begin{figure}[t]
  
\scalebox{.9}{
  \includegraphics[width=0.5\textwidth]{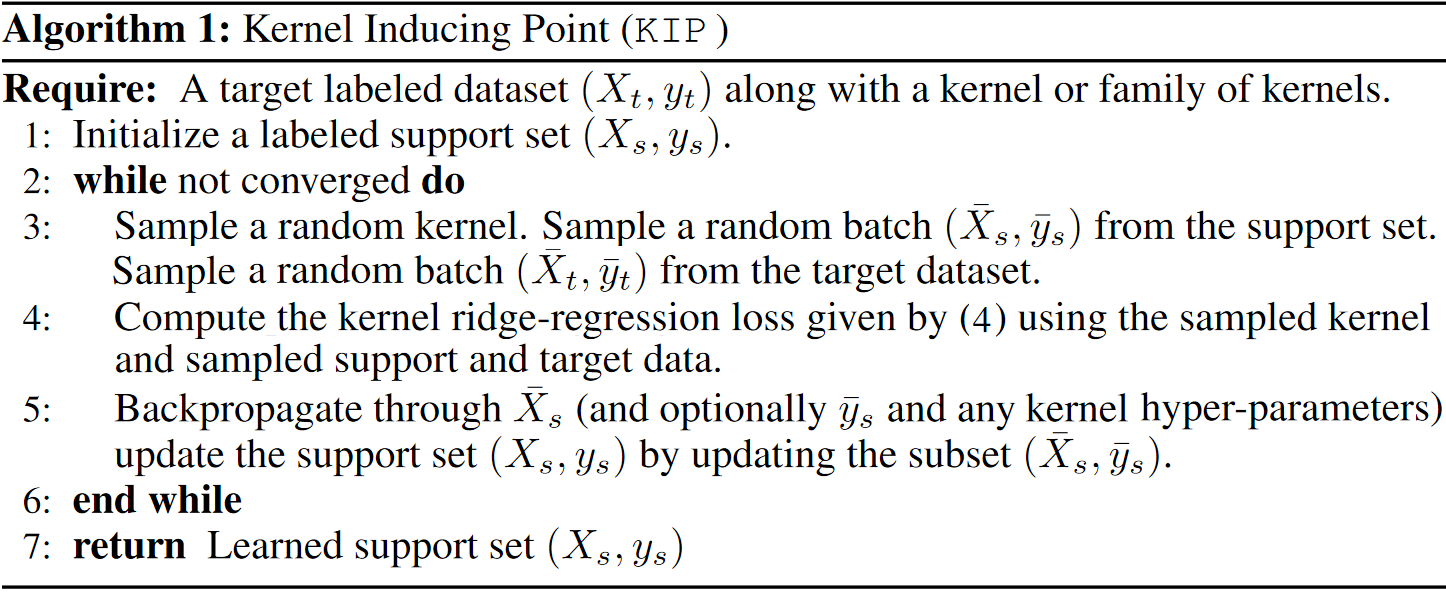}}
\end{figure}

\begin{figure*}[t]
\includegraphics[width=\textwidth]{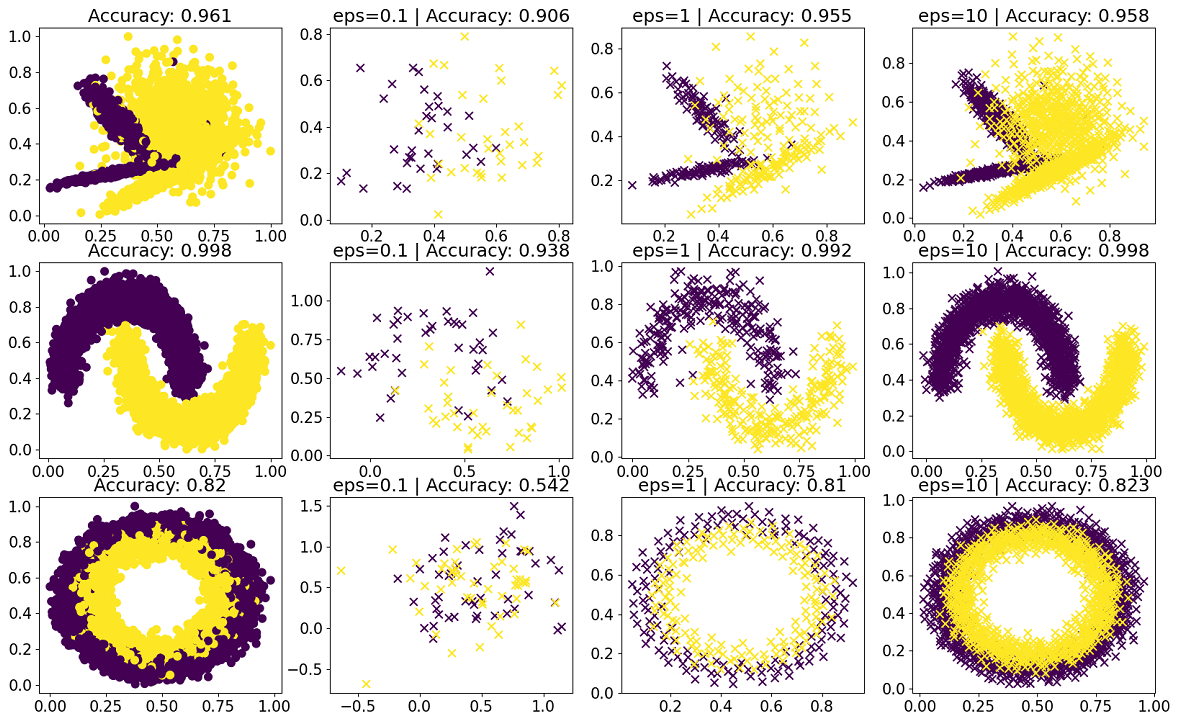}
\centering
\caption{Example synthetic data points generated with different epsilon parameters, on three toy datasets. 
}
\label{fig:toy}
\end{figure*}

\subsection{Anonymize through Gaussian Differential Privacy}


Once the data distribution is approximated and adapted, privacy-preserving synthetic data can be generated by a variety of approaches. \cite{lee2019synthesizing} showed that a simple approach of taking linear combinations of data points yields promising results, provided that a sufficient number of instances are mixed. We follow the same approach, with the difference that we mix data points of the same cluster, rather than random data points, in order to more closely preserve the underlying data distribution. Finally, we add noise as derived from Gaussian Differential Privacy as the last ingredient of the algorithm shown below.


To obtain the accuracies reported in the Results, Algorithm 2 was re-run with different parameter choices for $\sigma_{max}$. Of the resulting datasets, the one yielding the highest accuracy against the training set when training a classifier on the synthetic data was taken as the final synthetic dataset to evaluate.

\begin{figure}[ht]
  \label{alg:ClustMix}
  \scalebox{.95}{
  \includegraphics[width=0.5\textwidth]{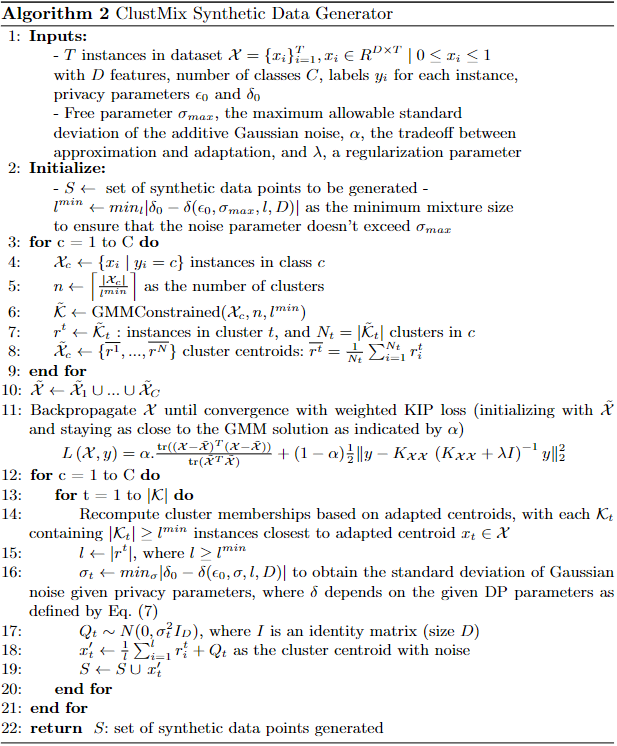}}

\end{figure}

\subsection{Differential privacy guarantees}
We first provide some preliminaries and theorems on differential privacy for our proofs. We adopt gaussian privacy since it achieved better tradeoff between privacy and utility \cite{https://doi.org/10.48550/arxiv.1911.11607}.
 \begin{definition} (Adjacent Datasets)
 let two datasets $S = {(X_{i}, Y_{i})}_{i=1}^{n}$ and $S^{'} = {(X_{i}^{'}, Y_{i}^{'})}_{i=1}^{n}$ be adjacent if they are identical except for one data point. We write $S \sim S^{'}$. \cite{dwork2006differential}
\end{definition}

 \begin{definition}  (Differential Privacy)
 A data publishing algorithm $M(S)$ is called $(\epsilon, \delta)$ differentially private if it satisfies, $$ P[M(S) \in A] \leq e^{\epsilon} P[M(S^{'}) \in A] + \delta
$$ M is an algorithm that publishes some statistics of the data. \cite{dwork2006differential}
  \end{definition}
Here you can consider $M(S)$ is an array where each entry means a feature in a generated data point. 

  \begin{definition}  (trade-off function) For any two probability distributions P and Q on the same
space, define the trade-off function $T(P, Q):[0,1] \to [0,1] $ as, 
$$T(P, Q)(\alpha) = inf\{\beta_{\phi}: \alpha_{\phi} \leq \alpha \}$$ where the infimum is taken over all measurable rejection rules.  $\alpha_{\phi} $ and $\beta_{\phi} $ represent the distributions. 
\cite{dong2019gaussian}
  
\end{definition}
   \begin{definition}  (Gaussian Differential Privacy)
 A data publishing algorithm $M(DD)$ is called $\mu-$differentially private if it satisfies, 
$$T(M(S), M(S^{'})) \geq G_{\mu}$$ where $G_{\mu}(\alpha) = \Phi(\Phi^{-1}(1 - \alpha) - \mu) $ where $\phi$ denotes standard normal CDF.
\cite{dong2019gaussian}
  \end{definition}

\begin{definition}  (Sensitivity)
Assume $\theta(S)$ is a univariate statistic of the dataset.  The sensitivity of $\theta$ is 
 $$sens(\theta) = sup_{S, S^{'}}|\theta(S) - \theta(S^{'})|  $$
 \cite{dwork2006differential}
  \end{definition}

   \begin{theorem}  
A Gaussian mechanism that operates on a statistic $\theta$ as $M(S)=\theta(S)+\xi$, where $\xi \sim \mathcal{N}\left(0, \operatorname{sens}(\theta)^{2} / \mu^{2}\right)$, is $\mu$-GDP.

  \end{theorem}
  \begin{proof}
See theorem 2.7 in \cite{dong2019gaussian}
\end{proof}
    \begin{corollary} 

A mechanism is $\mu$-GDP if and only if it is $(\epsilon,\delta(\epsilon))$-DP for all $\epsilon \geq 0$. Where, $$\delta(\varepsilon)=\Phi\left(-\frac{\varepsilon}{\mu}+\frac{\mu}{2}\right)-\mathrm{e}^{\varepsilon} \Phi\left(-\frac{\varepsilon}{\mu}-\frac{\mu}{2}\right)$$
\end{corollary}

\begin{proof}
See Corollary 2.13 in \cite{dong2019gaussian}
\end{proof}
    \begin{corollary} 

The n-fold composition of $\mu_{i}$-GDP mechanisms is $\sqrt{\mu_{1}^{2} + ... + \mu_{n}^{2}}$-GDP 
\end{corollary} 
\begin{proof}
See Corollary 3.3 in \cite{dong2019gaussian}
\end{proof}
   \begin{theorem} 
Consider T data points consisting of a feature matrix $X := [X_1X_2...X_n] \in R^{D \times T}$ and X is normalized such that $X \in [0,1]^{D \times T}$. We obtain each synthetic data point $X'_t$ by averaging $l$ original points, and adding Gaussian noise:

\begin{equation}
\label{eq:avg}
X'_t=\frac{1}{l}\sum_{i=1}^l X_{t_i} + Q_t, 
\end{equation}

where Qt is sampled from $N(0,\sigma_t^2 I_D)$ and $l \leq D$.

Suppose each original data point can only be used one time in the generation of the synthetic dataset and the number of category equals to C, this data publishing algorithm is $(\epsilon,\delta(\epsilon,\sigma,l,C, D))$-DP for all $\epsilon \leq 0$, where $\delta(\epsilon,\sigma,l,C,D)$ is given by 
\begin{equation}
\label{eq:epsilon}
\delta(\epsilon,\sigma,l,D) = \Phi(-\frac{\epsilon l \sigma}{\sqrt{CD}}+\frac{\sqrt{CD}}{2 l \sigma})-e^{\epsilon} \Phi(-\frac{\epsilon l \sigma}{\sqrt{CD}}-\frac{\sqrt{CD}}{2 l \sigma}), 
\end{equation}

and $\Phi$ denotes the standard normal CDF of $N(0,1)$. 

   \end{theorem} 
   \begin{proof}
If we change a single data point in the original data from $X_{1}$ to $X_{1}^{'}$, this will at most change a single synthetic data entry from $X_{t}$ to  $X_{t}^{'}$. By Definition 5 and Equation 6, we have $$sensitivity = sup|X_{t} - X_{t}^{'}| $$
$$ = sup|\frac{1}{l} (X_{1} -  \frac{1}{l} X_{1}^{'}) |$$ 
$$ \leq \frac{1}{l}$$ as each data point in $X \in [0, 1]^{D \times T}$. Thus, the sensitivity is $\frac{1}{l}$
   
Let $\mu=\frac{1}{l\sigma}$ and we have $Q_{t}$ is sampled from $N(0,\sigma_t^2 I_D)$ 
Since $$\sigma_t^2 = (\frac{l\sigma }{l})^{2} = (\frac{\frac{1}{l} }{\frac{1}{l \sigma}})^{2}$$ 
$$ = (\frac{\frac{1}{l}}{\mu})^{2}$$ Since the sensitivity is $\frac{1}{l}$ and based on theorem 1, we can show that our mechanism is  $\frac{1}{l\sigma}$-GDP.

Since changing one original data point can influence a maximum of one synthetic data point and D features, $\mu=\sqrt{1 / (l \sigma)^{2} \times D}=\frac{\sqrt{D}}{l \sigma}$. However, since we use KIP algorithm, one single point could potentially influence other classes in that region $\mu=\sqrt{1 / (l \sigma)^{2} \times D \times C}=\frac{\sqrt{CD}}{l \sigma}$. Since there are C classes,  Based on Corollary 2, our data publishing algorithm is thus $\frac{\sqrt{CD}}{l \sigma}$-GDP

Corollary 1 in \cite{dong2019gaussian} states that a mechanism is $\mu$-GDP if and only if it is $(\epsilon,\delta(\epsilon))$-DP for all $\epsilon \geq 0$. Here, $\delta(\varepsilon)=\Phi\left(-\frac{\varepsilon}{\mu}+\frac{\mu}{2}\right)-\mathrm{e}^{\varepsilon} \Phi\left(-\frac{\varepsilon}{\mu}-\frac{\mu}{2}\right)$, in order to ensure that a Gaussian output perturbation mechanism is $(\epsilon,\delta(\epsilon))$-DP, as first shown by \cite{balle2018improving}. Substituting $\mu=\frac{\sqrt{CD}}{l \sigma}$, our data publishing algorithm is thus $(\epsilon,\delta(\epsilon))$-DP, where $\delta(\epsilon)$
is given by $$\delta(\epsilon,\sigma,l,C,D) = \Phi(-\frac{\epsilon l \sigma}{\sqrt{CD}}+\frac{\sqrt{CD}}{2 l \sigma})-e^{\epsilon} \Phi(-\frac{\epsilon l \sigma}{\sqrt{CD}}-\frac{\sqrt{CD}}{2 l \sigma})$$
\end{proof}

Since each data point in the original set can be used by one randomly selected cluster to fit on GMM, and we only apply KIP algorithm to this region, it can only influence the generation of C sample. Thus, our data publishing algorithm is differentially private.

\begin{table*}[ht]
\scalebox{0.92}{
\begin{tabular}{l|ccc|cccccc}
\hline
\textit{}                                      &  num samples  & num classes & num features & Real  & DP-GAN & PATE-GAN & ADS-GAN & DP-MERF & ClustMix         \\ \hline
adult AUC                                               & 32,561 & 2 & 14  & 0.931 & 0.511  & 0.732    & 0.821   & 0.650   & \textbf{0.863} \\ \hline
census AUC                                               & 299,285 & 2 & 40 & 0.952 & 0.529  & 0.544    & 0.856   & 0.686   & \textbf{0.906} \\ \hline
credit AUC                                               & 284,807 & 2 & 29 & 0.973 & 0.435  & 0.739    & 0.911   & 0.772   & \textbf{0.969} \\ \hline
isolet AUC                                               & 4,366 & 2 & 617 & 0.988 & 0.618 & 0.529  &  \textbf{0.924}   & 0.547   & 0.903 \\ \hline
\begin{tabular}[c]{@{}l@{}}MIMIC3 EHR \\ Mortality AUC \end{tabular}                            & 58,976 & 2 & 18  & 0.896 & 0.632  & 0.473    & 0.530   & 0.728   & \textbf{0.749} \\ \hline
\begin{tabular}[c]{@{}l@{}}MIMIC3 EHR  \\ Length of stay AUC \end{tabular}                        & 57,171 & 4 & 25 & 0.728 & 0.497  & 0.501    & 0.550   & 0.508   & \textbf{0.624} \\ \hline
\begin{tabular}[c]{@{}l@{}}MNIST Accuracy \\ $\delta=10^{-5}$ \end{tabular}  & 60,000 & 10 & 784 & 0.972 & 0.403  & 0.563    & 0.710       & 0.650   & \textbf{0.818} \\ \hline \hline
\begin{tabular}[c]{@{}l@{}}Avg. score difference \;  \\ compared to real \end{tabular}                & & & & 0 & 0.402 & 0.337 & 0.162 & 0.271 & \textbf{0.087} \\ \hline
\end{tabular} }
\caption{Results of evaluating classifiers trained on synthetic data and tested on real held-out test sets. DP models employ (1,1/N)-DP with N being sample size, except for MNIST, where $\delta=10^{-5}$ (following the authors of the compared approaches). Length of Stay was discretized into 4 categories ($\leq 3$ days, $\leq 1$ week, $ \leq2$ weeks or $\geq 2$ weeks) to facilitate classification, and evaluated by one vs. one AUC}
\label{tab:results}
\end{table*}

\subsection{Data and Experimental Setup}


\textbf{Implementation Details}: we leverage an extension of Gaussian Mixture Model implementation of scikit-learn for the Approximate step (adapted to ensure clusters never contain fewer data points than the given constraint - see Section 3.1; reverting to constrained KMeans at timeout or if no solution was found), and an open-source implementation of KIP\footnote{https://github.com/google-research/google-research/tree/master/kip} for the Adapt step, modified to optimize for the weighted objective function in Eq. \eqref{eq:kip}, trading off faithfulness to the output of the Approximate step and the utility for ML by means of the $\alpha$ parameter. Noise addition is based on GDP; we obtain the standard deviation of the Gaussian to sample noise from Eq. \eqref{eq:epsilon} by means of optimization, as it does not have a closed-form solution. Specifically, we use Nelder-Mead procedure \cite{olsson1975nelder} to obtain the smallest possible noise that can be added that ensures that given DP parameters $\epsilon$ and $\delta$ are not exceeded (and discard generated data points where no such solution can be found). Obtained noise parameters were cached for computational performance.


\textbf{Data preprocessing}: data features were scaled to the interval [0,1] for ClustMix to simplify proportionate noise addition (using min-max scaling or a Normal quantile transform, whichever worked better during hyperparameter optimization). For the only dataset with continuous target variables (MIMIC-III hospital length of stay), lengths of stay were discretized into 4 buckets for classification (less than 3 days, less than a week, less than two weeks, or two weeks or more). 

\textbf{Hyperparameter tuning}: for all benchmark algorithms as well as for ClustMix, we use heteroscedastic evolutionary Bayesian optimisation (HEBO) \cite{cowen2020hebo} to optimize free parameters - over 30 iterations for all approaches - excepting DP parameters which are fixed ($\epsilon=1$ and $\delta=1/N$ with N being sample size; except for MNIST, where $\delta=10^{-5}$ to facilitate easier comparison with benchmark papers using the same approach). 

\textbf{Benchmark Datasets and Methods}: We compare with state-of-the-art DP data generation algorithms such as DP-GAN, DP-MERF, PATE-GAN. We also included non DP data generation algorithm ADS-GAN as it can lower the risk of information leakage. We compare on 7 different real datasets which we list deatils of datasets in Table 2.

\textbf{Evaluation}: qualitative results (i.e. MNIST images, and plots of toy data) can be inspected in Appendix A. Quantitative results are presented in Table \ref{tab:results} using Area Under the ROC Curve (AUC), micro-averaged (i.e. one vs. one for all pairwise combinations) in the case of multi-class classification, except for MNIST, conventionally evaluated using accuracy. For each dataset, we set aside $25\%$ of the data as a held-out test set (except MNIST, which has its own conventionally used test set, comprising $14\%$ of the total data). `Real' performance metrics are calculated by training a LightGBM model of gradient boosted decision trees \cite{ke2017lightgbm} on the training set and evaluating on the test set. For each benchmark method and for ClustMix, performance metrics are calculated by first running the method to generate a privacy-preserving synthetic training set based on the real training set, then training a LightGBM model on that synthetic training set, and evaluating it on the real test set.

\section{Results}

\begin{table}[]
\label{datad}
\centering
\begin{tabular}{llll}
\hline
       & \# samples & \# classes & \# features \\ \hline
adult  & 32,561     & 2          & 14          \\ \hline
census & 299,285    & 2          & 40          \\ \hline
credit & 284,807    & 2          & 29          \\ \hline
isolet & 4366       & 2          & 617         \\ \hline
MIMIC3 mortality & 58,976       & 2          & 18         \\ \hline
MIMIC3 length of stay & 57,171       & 4          & 25         \\ \hline
MNIST  & 60,000     & 10         & 784         \\ \hline
\end{tabular}
\caption{Dataset characteristics}
\label{tab:datasetchar}
\end{table}

We evaluate ClustMix against several states of the art differentially private data generation approaches in Table \ref{tab:results}. The combination of cluster-based instead of random mixing for preserving differential privacy (unlike similar mixing-based approaches \cite{lee2018sgd,lee2019synthesizing}), as well as adaptation for machine learning utility, result in significantly higher predictive accuracy of models trained on synthetic and tested on real data, when compared against differentially private methods. 

The mechanism can be more clearly understood from the simple 2D toy datasets illustrated in Figure \ref{fig:toy}. In an attempt to pick $\sigma_{max}$ that maximizes utility on the training set, ClustMix is forced to generate a very small number of noisy cluster centroids when $\epsilon=0.1$ (which ensures large $l$ and thus minimizes the additive noise), but yields many low-noise centroids at $\epsilon=10$ (since at this $\epsilon$, smaller clusters are sufficient to ensure that the noise levels are small enough).

Similar observations apply to the higher-dimensional MNIST digit dataset (Figure \ref{fig:mnist}), where larger clusters yield the optimal accuracy at $\epsilon=1$ (each digit is a linear combination of 118 training digits on average, with added noise), but smaller clusters are more helpful at $\epsilon=0.2$ (where each digit is obtained by averaging 20 training digits plus noise).


\section{Conclusion}

We propose a simple framework for privacy-preserving synthetic data generation that explicitly takes into account machine learning utility, in addition to faithfully modeling data distribution. We have also presented a specific instantiation of this framework, ClustMix, and have presented substantial increases in classification performance metrics compared to state-of-the-art models; implying that the presented framework constitutes a promising direction of research increasing the utility of low-risk synthetic data release for machine learning. \\

\appendix
\newpage
\section{Appendix}

\section{MNIST Example}

\begin{figure*}[t]
\includegraphics[width=\textwidth]{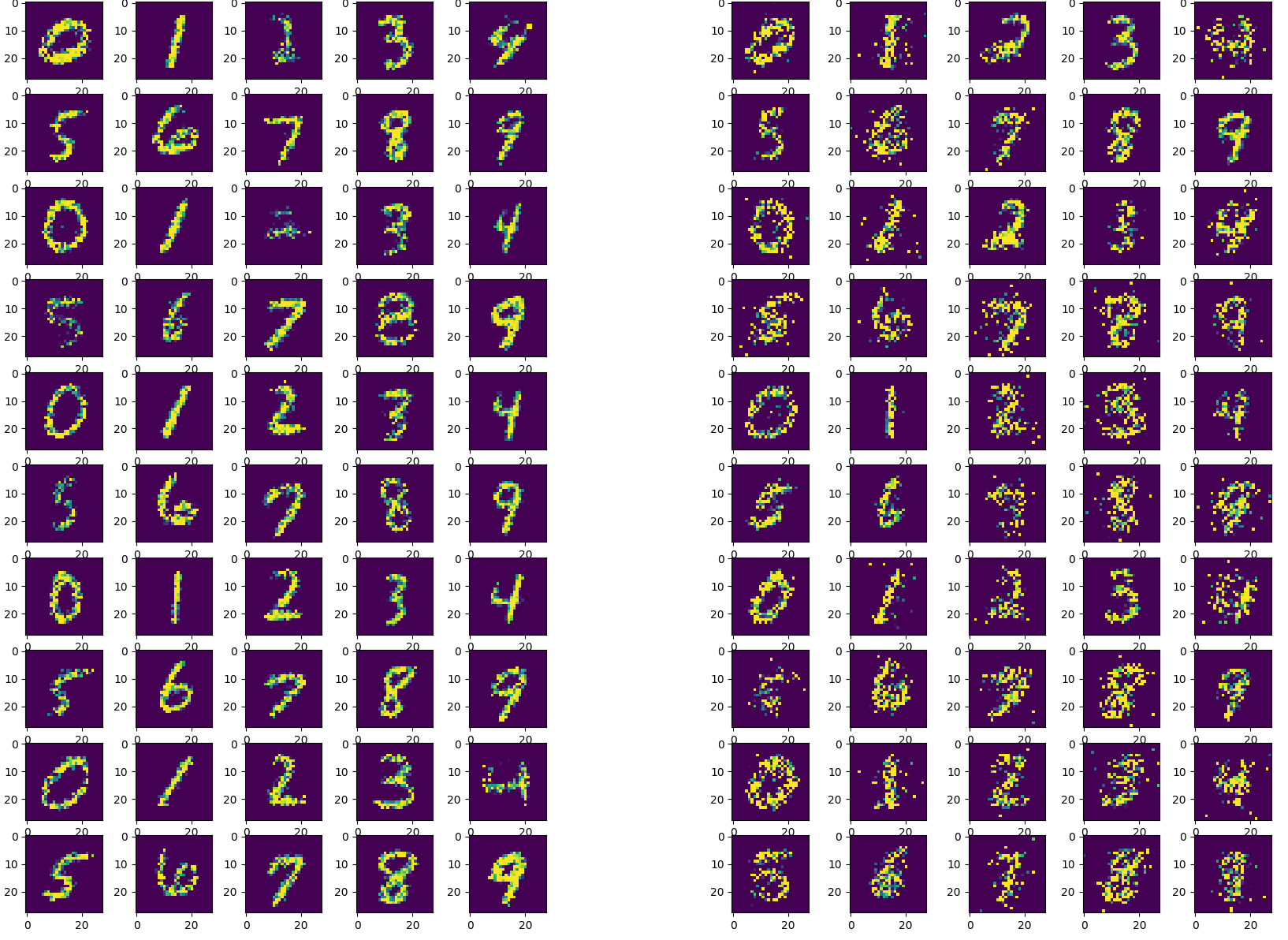}
\centering
\caption{Example synthetic data points generated with $\epsilon$ and $\delta$ parameters (1,$10^{-5}$) in the left plot and (0.2,$10^{-5}$) in the right, on the 784-dimensional MNIST dataset. Models trained on this data and tested on held-out real images perform at 0.818 (at $\epsilon=1$) and 0.650 (at $\epsilon=0.2$) accuracy respectively, compared to 0.972 accuracy on the real dataset with different cluster sizes.}
\label{fig:mnist}
\end{figure*}


\subsubsection{Acknowledgements} Many thanks to Cemre Zor for helpful comments on the Manuscript.
\newpage

\bibliography{aaai23}

\begin{thebibliography}{36}
\providecommand{\natexlab}[1]{#1}

\bibitem[{Agrawal and Srikant(2000)}]{agrawal2000privacy}
Agrawal, R.; and Srikant, R. 2000.
\newblock Privacy-preserving data mining.
\newblock In \emph{Proceedings of the 2000 ACM SIGMOD international conference
  on Management of data}, 439--450.

\bibitem[{Balle and Wang(2018)}]{balle2018improving}
Balle, B.; and Wang, Y.-X. 2018.
\newblock Improving the Gaussian mechanism for differential privacy: Analytical
  calibration and optimal denoising.
\newblock In \emph{International Conference on Machine Learning}, 394--403.
  PMLR.

\bibitem[{Bishop(2006)}]{bishop2006pattern}
Bishop, C.~M. 2006.
\newblock Pattern recognition.
\newblock \emph{Machine learning}, 128(9).

\bibitem[{Blum, Ligett, and Roth(2013)}]{blum2013learning}
Blum, A.; Ligett, K.; and Roth, A. 2013.
\newblock A learning theory approach to noninteractive database privacy.
\newblock \emph{Journal of the ACM (JACM)}, 60(2): 1--25.

\bibitem[{Bu et~al.(2019)Bu, Dong, Long, and
  Su}]{https://doi.org/10.48550/arxiv.1911.11607}
Bu, Z.; Dong, J.; Long, Q.; and Su, W.~J. 2019.
\newblock Deep Learning with Gaussian Differential Privacy.

\bibitem[{Chapelle et~al.(2001)Chapelle, Weston, Bottou, and
  Vapnik}]{chapelle2001vicinal}
Chapelle, O.; Weston, J.; Bottou, L.; and Vapnik, V. 2001.
\newblock Vicinal risk minimization.
\newblock \emph{Advances in neural information processing systems}, 416--422.

\bibitem[{Cohen and Nissim(2020)}]{cohen2020towards}
Cohen, A.; and Nissim, K. 2020.
\newblock Towards formalizing the GDPR’s notion of singling out.
\newblock \emph{Proceedings of the National Academy of Sciences}, 117(15):
  8344--8352.

\bibitem[{Cowen-Rivers et~al.(2020)Cowen-Rivers, Lyu, Wang, Tutunov, Jianye,
  Wang, and Ammar}]{cowen2020hebo}
Cowen-Rivers, A.~I.; Lyu, W.; Wang, Z.; Tutunov, R.; Jianye, H.; Wang, J.; and
  Ammar, H.~B. 2020.
\newblock Hebo: Heteroscedastic evolutionary bayesian optimisation.
\newblock \emph{arXiv e-prints}, arXiv--2012.

\bibitem[{Dong, Roth, and Su(2019)}]{dong2019gaussian}
Dong, J.; Roth, A.; and Su, W.~J. 2019.
\newblock Gaussian differential privacy.
\newblock \emph{arXiv preprint arXiv:1905.02383}.

\bibitem[{Dwork(2006)}]{dwork2006differential}
Dwork, C. 2006.
\newblock Differential privacy.
\newblock In \emph{International Colloquium on Automata, Languages, and
  Programming}, 1--12. Springer.

\bibitem[{Goodfellow et~al.(2014)Goodfellow, Pouget-Abadie, Mirza, Xu,
  Warde-Farley, Ozair, Courville, and Bengio}]{goodfellow2014generative}
Goodfellow, I.; Pouget-Abadie, J.; Mirza, M.; Xu, B.; Warde-Farley, D.; Ozair,
  S.; Courville, A.; and Bengio, Y. 2014.
\newblock Generative adversarial nets.
\newblock \emph{Advances in neural information processing systems}, 27.

\bibitem[{Hardt, Ligett, and McSherry(2012)}]{hardt2012simple}
Hardt, M.; Ligett, K.; and McSherry, F. 2012.
\newblock A simple and practical algorithm for differentially private data
  release.
\newblock \emph{Advances in neural information processing systems}, 25.

\bibitem[{Inoue(2018)}]{inoue2018data}
Inoue, H. 2018.
\newblock Data augmentation by pairing samples for images classification.
\newblock \emph{arXiv preprint arXiv:1801.02929}.

\bibitem[{Jacot, Gabriel, and Hongler(2018)}]{jacot2018neural}
Jacot, A.; Gabriel, F.; and Hongler, C. 2018.
\newblock Neural tangent kernel: Convergence and generalization in neural
  networks.
\newblock \emph{Advances in neural information processing systems}, 31.

\bibitem[{Jitta and Klami(2018)}]{jitta2018controlling}
Jitta, A.; and Klami, A. 2018.
\newblock On controlling the size of clusters in probabilistic clustering.
\newblock In \emph{Thirty-Second AAAI Conference on Artificial Intelligence}.

\bibitem[{Jordon, Yoon, and Van Der~Schaar(2018)}]{jordon2018pate}
Jordon, J.; Yoon, J.; and Van Der~Schaar, M. 2018.
\newblock PATE-GAN: Generating synthetic data with differential privacy
  guarantees.
\newblock In \emph{International Conference on Learning Representations}.

\bibitem[{Karakus et~al.(2017)Karakus, Sun, Diggavi, and
  Yin}]{karakus2017straggler}
Karakus, C.; Sun, Y.; Diggavi, S.; and Yin, W. 2017.
\newblock Straggler mitigation in distributed optimization through data
  encoding.
\newblock \emph{Advances in Neural Information Processing Systems}, 30:
  5434--5442.

\bibitem[{Kasiviswanathan et~al.(2011)Kasiviswanathan, Lee, Nissim,
  Raskhodnikova, and Smith}]{kasiviswanathan2011can}
Kasiviswanathan, S.~P.; Lee, H.~K.; Nissim, K.; Raskhodnikova, S.; and Smith,
  A. 2011.
\newblock What can we learn privately?
\newblock \emph{SIAM Journal on Computing}, 40(3): 793--826.

\bibitem[{Ke et~al.(2017)Ke, Meng, Finley, Wang, Chen, Ma, Ye, and
  Liu}]{ke2017lightgbm}
Ke, G.; Meng, Q.; Finley, T.; Wang, T.; Chen, W.; Ma, W.; Ye, Q.; and Liu,
  T.-Y. 2017.
\newblock Lightgbm: A highly efficient gradient boosting decision tree.
\newblock \emph{Advances in neural information processing systems}, 30.

\bibitem[{Kenthapadi et~al.(2012)Kenthapadi, Korolova, Mironov, and
  Mishra}]{kenthapadi2012privacy}
Kenthapadi, K.; Korolova, A.; Mironov, I.; and Mishra, N. 2012.
\newblock Privacy via the johnson-lindenstrauss transform.
\newblock \emph{arXiv preprint arXiv:1204.2606}.

\bibitem[{Lee et~al.(2019)Lee, Kim, Lee, Suh, and
  Ramchandran}]{lee2019synthesizing}
Lee, K.; Kim, H.; Lee, K.; Suh, C.; and Ramchandran, K. 2019.
\newblock Synthesizing differentially private datasets using random mixing.
\newblock In \emph{2019 IEEE International Symposium on Information Theory
  (ISIT)}, 542--546. IEEE.

\bibitem[{Lee et~al.(2018)Lee, Lee, Kim, Suh, and Ramchandran}]{lee2018sgd}
Lee, K.; Lee, K.; Kim, H.; Suh, C.; and Ramchandran, K. 2018.
\newblock SGD on Random Mixtures: Private Machine Learning under Data Breach
  Threats.
\newblock \emph{ICLR 2018 Workshop}.

\bibitem[{McLachlan, Lee, and Rathnayake(2019)}]{mclachlan2019finite}
McLachlan, G.~J.; Lee, S.~X.; and Rathnayake, S.~I. 2019.
\newblock Finite mixture models.
\newblock \emph{Annual review of statistics and its application}, 6: 355--378.

\bibitem[{Mishra and Sandler(2006)}]{mishra2006privacy}
Mishra, N.; and Sandler, M. 2006.
\newblock Privacy via pseudorandom sketches.
\newblock In \emph{Proceedings of the twenty-fifth ACM SIGMOD-SIGACT-SIGART
  symposium on Principles of database systems}, 143--152.

\bibitem[{Mohammed et~al.(2011)Mohammed, Chen, Fung, and
  Yu}]{mohammed2011differentially}
Mohammed, N.; Chen, R.; Fung, B.~C.; and Yu, P.~S. 2011.
\newblock Differentially private data release for data mining.
\newblock In \emph{Proceedings of the 17th ACM SIGKDD international conference
  on Knowledge discovery and data mining}, 493--501.

\bibitem[{Nguyen, Chen, and Lee(2020)}]{nguyen2020dataset}
Nguyen, T.; Chen, Z.; and Lee, J. 2020.
\newblock Dataset Meta-Learning from Kernel Ridge-Regression.
\newblock In \emph{International Conference on Learning Representations}.

\bibitem[{Nguyen et~al.(2021)Nguyen, Novak, Xiao, and Lee}]{nguyen2021dataset}
Nguyen, T.; Novak, R.; Xiao, L.; and Lee, J. 2021.
\newblock Dataset distillation with infinitely wide convolutional networks.
\newblock \emph{Advances in Neural Information Processing Systems}, 34.

\bibitem[{Olsson and Nelson(1975)}]{olsson1975nelder}
Olsson, D.~M.; and Nelson, L.~S. 1975.
\newblock The Nelder-Mead simplex procedure for function minimization.
\newblock \emph{Technometrics}, 17(1): 45--51.

\bibitem[{Papernot et~al.(2016)Papernot, Abadi, Erlingsson, Goodfellow, and
  Talwar}]{papernot2016semi}
Papernot, N.; Abadi, M.; Erlingsson, U.; Goodfellow, I.; and Talwar, K. 2016.
\newblock Semi-supervised knowledge transfer for deep learning from private
  training data.
\newblock \emph{arXiv preprint arXiv:1610.05755}.

\bibitem[{Papernot et~al.(2018)Papernot, Song, Mironov, Raghunathan, Talwar,
  and Erlingsson}]{papernot2018scalable}
Papernot, N.; Song, S.; Mironov, I.; Raghunathan, A.; Talwar, K.; and
  Erlingsson, {\'U}. 2018.
\newblock Scalable private learning with pate.
\newblock \emph{arXiv preprint arXiv:1802.08908}.

\bibitem[{Xie et~al.(2018)Xie, Lin, Wang, Wang, and
  Zhou}]{xie2018differentially}
Xie, L.; Lin, K.; Wang, S.; Wang, F.; and Zhou, J. 2018.
\newblock Differentially private generative adversarial network.
\newblock \emph{arXiv preprint arXiv:1802.06739}.

\bibitem[{Xu et~al.(2017)Xu, Ren, Zhang, Qin, and Ren}]{xu2017dppro}
Xu, C.; Ren, J.; Zhang, Y.; Qin, Z.; and Ren, K. 2017.
\newblock DPPro: Differentially private high-dimensional data release via
  random projection.
\newblock \emph{IEEE Transactions on Information Forensics and Security},
  12(12): 3081--3093.

\bibitem[{Yoon, Drumright, and Van Der~Schaar(2020)}]{yoon2020anonymization}
Yoon, J.; Drumright, L.~N.; and Van Der~Schaar, M. 2020.
\newblock Anonymization through data synthesis using generative adversarial
  networks (ads-gan).
\newblock \emph{IEEE journal of biomedical and health informatics}, 24(8):
  2378--2388.

\bibitem[{Zhang et~al.(2017{\natexlab{a}})Zhang, Cisse, Dauphin, and
  Lopez-Paz}]{zhang2017mixup}
Zhang, H.; Cisse, M.; Dauphin, Y.~N.; and Lopez-Paz, D. 2017{\natexlab{a}}.
\newblock mixup: Beyond empirical risk minimization.
\newblock \emph{arXiv preprint arXiv:1710.09412}.

\bibitem[{Zhang et~al.(2017{\natexlab{b}})Zhang, Cormode, Procopiuc,
  Srivastava, and Xiao}]{zhang2017privbayes}
Zhang, J.; Cormode, G.; Procopiuc, C.~M.; Srivastava, D.; and Xiao, X.
  2017{\natexlab{b}}.
\newblock Privbayes: Private data release via bayesian networks.
\newblock \emph{ACM Transactions on Database Systems (TODS)}, 42(4): 1--41.

\bibitem[{Zhu et~al.(2017)Zhu, Li, Zhou, and Philip}]{zhu2017differentially}
Zhu, T.; Li, G.; Zhou, W.; and Philip, S.~Y. 2017.
\newblock Differentially private data publishing and analysis: A survey.
\newblock \emph{IEEE Transactions on Knowledge and Data Engineering}, 29(8):
  1619--1638.

\end{thebibliography}
\end{document}